\documentclass{article} 
\usepackage[preprint]{colm2026_conference}

\usepackage{microtype}
\usepackage{hyperref}
\usepackage{url}
\usepackage{booktabs}

\usepackage{hyperref}

\usepackage{amsmath}
\usepackage{amssymb}
\usepackage{mathtools}
\usepackage{amsthm}
\usepackage{enumitem}

\usepackage[T1]{fontenc}    
\usepackage{hyperref}       
\usepackage{url}            
\usepackage{booktabs}       
\usepackage{amsfonts}       
\usepackage{nicefrac}       
\usepackage{microtype}      
\usepackage{xcolor}         

\usepackage{framed}
\usepackage{listings}

\usepackage{algorithmic}
\usepackage{algorithm}
\usepackage{subcaption}
\usepackage{array}

\usepackage[capitalize,noabbrev]{cleveref}

\theoremstyle{plain}
\newtheorem{theorem}{Theorem}[section]

\newtheorem{lemma}[theorem]{Lemma}

\theoremstyle{definition}

\theoremstyle{remark}

\usepackage[
disable,
textsize=tiny]{todonotes}
\usepackage{xspace}

\newcommand{\todonev}[2][]{\todo[noinline,size=\scriptsize,color=pink!20!white,#1]{N: #2}\xspace}
\newcommand{\todoa}[2][]{\todo[noinline,size=\scriptsize,color=green!20!white,#1]{A: #2}\xspace}

\newcommand{\cD}{\mathcal D}
\newcommand{\cP}{\mathcal P}

\newcommand{\cV}{\mathcal V}

\newcommand{\sm}{\text{softmax}}
\newcommand{\attn}{\text{attn}}
\newcommand{\cattn}{\text{copyattn}}
\newcommand{\lnorm}{\text{layernorm}}
\newcommand{\mlp}{\text{MLP}}

\usepackage{lineno}

\definecolor{darkblue}{rgb}{0, 0, 0.5}
\hypersetup{colorlinks=true, citecolor=darkblue, linkcolor=darkblue, urlcolor=darkblue}

\title{
To See the Unseen: on the Generalization Ability of Transformers in Symbolic Reasoning

}

\author{Nevena Lazi\'c
\And
Liam Fowl
\And
Andr\'as Gy\"orgy 
\And
Csaba Szepesv\'ari
}

\begin{document}

\ifcolmsubmission
\linenumbers
\fi

\maketitle

\begin{abstract}
We investigate the ability of decoder-only transformer models to perform abstract symbolic reasoning; specifically solving propositional logic reasoning problems given in-context. 
Previous work demonstrated that models fail to generalize to problems involving variable names that were not observed during training, and 
it was shown that one reason behind this 
is the difficulty of copying (or generating) unseen tokens. 
We show both theoretically and empirically that a particular representational collapse also has a crucial role: the \emph{unembeddings} (last-layer weights) of unseen tokens collapse to nearly the same vector during training. The collapse makes distinguishing multiple unseen variables difficult for the model (especially when the embedding and unembedding parameters are shared), and provides a mechanistic explanation for the effectiveness of existing heuristic interventions like "active forgetting", which periodically reset the token (un)embeddings. Based on these observations, we devise a combination of techniques, involving a small architecture change facilitating copying, data diversity, and freezing or resetting (un)embeddings, that achieves generalization to unseen tokens. We support our claims with extensive controlled experiments on propositional logic reasoning problems. 
Beyond synthetic experiments, we also observe evidence of (un)embedding collapse in the open-weight models in the Gemma 3 family, which includes 99 unused tokens reserved for downstream use. Empirically we find that the correlated embeddings of these tokens are a poor initialization for finetuning applications.
\end{abstract}

\section{Introduction}

Modern transformer-based large language models (LLMs) \citep{vaswani2017attention} trained using next-token prediction have achieved significant success across a wide range of tasks, including benchmarks requiring formal reasoning such as mathematics and science. At the same time, researchers continue to uncover simple ways in which transformers fail on reasoning problems -- for example, by changing names and numerical values or inserting irrelevant information \citep{jiang2024peek,mirzadehgsm}.  This fragility suggests that LLMs are not yet capable of fully general reasoning and can rely on superficial patterns to produce answers.

In this work, we investigate the ability of decoder-only transformer models to perform abstract symbolic reasoning. Specifically, we study the models' ability to execute in-context algorithms on inputs containing arbitrary new tokens as symbols (see Figure~\ref{fig:simple_logic} for an example). Thus, the models need to generalize solely based on the problem structure and not based on the content of token embeddings. 
This problem was studied by \citet{boixcan} in the neural tangent kernel (NTK) regime, with only the (un)embedding\footnote{
We consider transformers with tied \emph{embedding} and \emph{unembedding} (first- and last-layer, resp.) matrices, which we will refer to as the (un)embedding parameters. Here the same parameter matrix is used for the embedding and unembedding layers, the latter being a transpose of the former. We will refer to these parameters as the \emph{(un)embedding}. 
This architecture choice is used in \citet{boixcan} and also commonly used in practice, for example in the Gemma 3 models \citep{team2025gemma}.} 
parameters trainable. They showed that transformers can learn to reason with unseen tokens given highly diverse training data, 
but fail at tasks that require copying (or generating) unseen tokens, and proposed a simple architecture change to improve copying capabilities. 
\citet{ananddual} studied generalization to unseen tokens empirically 
and found that periodically resetting the embeddings (or "active forgetting") during training helps, but did not provide an 
explanation beyond not storing information in the embeddings. Furthermore, they did not consider tasks requiring symbolic copying, which remains difficult with their approach. 

Our work identifies a key issue in generalization to unseen tokens that has been overlooked by previous work. We find that the (un)embeddings of the unseen tokens \emph{collapse} to nearly the same vector during the course of training. This leads transformers to struggle when unseen tokens are used to represent different variable names, as the corresponding variables become nearly indistinguishable. 
We prove theoretically that this collapse is inevitable when training transformers using SGD with weight decay and layernorm, and provide empirical evidence of it occurring in practice. Our work shows that the architecture change proposed by \citet{boixcan} is ineffective on its own in the presence of multiple unseen variables. The identified embedding collapse also explains 
the usefulness of the periodic resetting strategy of \citet{ananddual}. 
We also show that reliable symbolic reasoning can be achieved using a combination of a copy-enabled architecture, high symbolic diversity in the data, and either freezing or periodically resetting (un)embeddings to prevent collapse. 
We validate our findings using extensive controlled experiments on a testbed of synthetic propositional logic problems similar to the examples in Figure~\ref{fig:simple_logic}.

Beyond experiments on synthetic data, we also investigate the implications of our findings on larger open-weights models. 
We find that (un)embedding collapse is also present in 
the Gemma 3 model family \citep{team2025gemma}, which has 99 deliberately unused tokens for downstream use. We demonstrate empirically that the correlated embeddings of these tokens are a poor initialization for downstream finetuning applications.

\begin{figure}[t!]
\small
\begin{subfigure}{.5\textwidth}
\begin{framed} 
{\color{blue} {\bf Rules:} If A then B.  {\bf Facts:} A. 
{\bf Query:} B ? } \newline
{\color{magenta} {\bf Reasoning:} 
 Facts: A. 
 If A then B. Facts: A, B.
 }\!\!\!\!\newline
 {\color{teal} {\bf Answer:} yes}
 \end{framed}
 \end{subfigure}
 \begin{subfigure}{.5\textwidth}
\begin{framed} 
{\color{blue} {\bf Rules:}
 If A then U.  {\bf Facts:} A. 
 {\bf Query:} U? } \newline  \newline
{\color{blue} {\bf Rules:}
 If U1 then U2.  {\bf Facts:} U1. 
 {\bf Query:} U2? }
 \end{framed}
 \end{subfigure}
    \caption{{\bf Left}: An example of a propositional logic problem with reasoning used for training. Given a set of \emph{rules} (definite clauses), a set of \emph{facts} (true zero-order predicates), and a query predicate, the goal is to compute the truth value of the query. The reasoning trace (shown above) executes the forward-chaining algorithm, and stops if the query is proven true. {\bf Right}: Two types of test examples. At test time, we use unseen symbols (denoted here by U, U1, U2) for either just the query variable, or for all variable names.}
    \label{fig:simple_logic}
\end{figure}

\section{Related Work}

\paragraph{Symbolic reasoning} Several previous works have demonstrated that transformers suffer from \emph{token bias} \citep{jiang2024peek}: perturbing certain input tokens can induce predictable changes in the LLM output. \citet{jiang2024peek} show that changing names, inserting celebrity references and irrelevant content, and replacing quantifiers with synonyms degrades LLM performance on logical fallacies. \citet{mirzadehgsm} introduce a version of the GSM8k mathematical reasoning benchmarks with modified numerical values and inserted irrelevant information. \citet{mccoy2024embers,mccoy2024language} show that the accuracy of LLMs is heavily influenced by the likelihood of task formulations, inputs, or outputs appearing in the training data. All these works suggest that transformers can rely on statistical shortcuts rather than performing true reasoning. 
\citet{boixcan} study the ability of transformers to reason with abstract symbols on relational tasks, in the neural tangent kernel (NTK) limit. They show that on regression-type tasks (computing a numerical label based on a symbol pattern), transformers can generalize to unseen tokens when trained on data with a large amount of symbolic substitution. On the other hand, they show that tasks involving generating an unseen token (such as copying from context) are considerably more difficult and may require architectural changes. 
\citet{ananddual} define \emph{structural in-context learning}  as the ability of a model to execute in-context learning (ICL) on arbitrary novel tokens, which is similar to symbolic reasoning. They study structural ICL on simple part-of-speech classification problems, and find that it is transient. They observe that it can be maintained with an \emph{active forgetting} strategy \citep{chen2023improving}, which resets the embedding matrix every $k$ steps, and propose \emph{temporary active forgetting}, where this is only done at the beginning of training. However, their work does not explain any mechanisms behind the failure.
\todonev[inline]{ add maybe: 
Multi-head Transformers Provably Learn Symbolic Multi-step Reasoning via Gradient Descent
}
\vspace{-1em}

\paragraph{Copying architectures} Multiple works explore neural architectures that enable copying of inputs, where attending to input tokens increases the probability of copying them. \citet{ontanon2022making} mix a distribution over input tokens with the standard transformer predictive distribution and find that this is helpful in solving compositional tasks that involve symbolic reasoning.  \citet{boixcan} mix the copying component at the level of logits, and our work studies a similar architectural change. Similar copying mechanisms to these have previously been studied in sequence-to-sequence models, starting with the \emph{pointer networks} of \citet{vinyals2015pointer} \citep[see, e.g., ][]{see2017get, gu2016incorporating}.

\paragraph{In-context learning (ICL)}  Token bias can partially be attributed to models inappropriately using in-weights knowledge in place of context information. A recent line of works examines the effect of distributional properties of the training data on the emergence of ICL  \citep{chan2022data,reddy2023mechanistic,singh2023transient,chantoward2025}. These works find that ICL can emerge in transformers when there is a large number of classes and large within-class variation. \citet{chan2022data,chan2022transformers} further observe that ICL and in-weight learning (IWL) can emerge simultaneously when the distribution over classes is Zipfian. \citet{chantoward2025} propose a simplified model that uses a gating mechanism to choose between IWL and ICL predictors based on test error. Since the test error is unavailable, it is approximated by the best choice in hindsight based on the difference between the losses of the two predictors.  Multiple works have noticed that ICL can become transient in an asymptotic training regime \citep{chan2022data,chan2022transformers,chantoward2025,panwarcontext}.
In the setting of ICL for ridge regression, \citet{raventos2023pretraining} show that models trained on a small number of tasks behave like the Bayesian estimator over the training-task distribution, but is able to generalize to unseen tasks (implementing essentially ridge regression) when the number of training tasks is sufficiently large. 
\citet{he2024learning} study ICL for modular arithmetic and find that models transition from in-distribution to out-of-distribution generalization as the number of pretraining tasks increases. A common observation in these papers is that data diversity is helpful for ICL, and we observe the same in our problems.
\vspace{-0.5em}


\paragraph{Chain-of-thought (CoT) reasoning}
CoT reasoning enables LLMs to generate intermediate computational steps prior to finalizing the answer to a question. CoT has played a key role in LLMs achieving exceptional performance on tasks such as mathematical problem-solving and code generation \citep{chowdhery2023palm,achiam2023gpt,anil2023palm,trinh2024solving,romera2024mathematical}.  Here, we study problems in which CoT reasoning requires computing quantities based on previously unseen symbolic variables, as well as referring to unseen symbolic variables by copying them. 
\vspace{-0.5em}

\paragraph{Randomized embeddings} Randomized position embeddings have been shown to help with length generalization in arithmetic tasks \citep{ruoss2023randomized,shen2023positional}. This is because conventional positional embeddings are out-of-distribution when increasing the sequence length beyond the training-time window. Our work shows that using randomized embeddings for all tokens can be helpful in generalization to unseen tokens.
\section{Transformer architecture}
\label{sec:architecture}

Given a vocabulary $\cV$, a decoder-only transformer with parameters $\theta$ maps a sequence of tokens $(x_1, \ldots, x_n) \in \cV^n$ to a probability distribution over $\cV$, denoted by $p_\theta(\cdot |x_1, \ldots, x_n)$.  Let $X$ be an $n \times |\cV|$ matrix of stacked indicator vectors for the input tokens. The architecture of an $L$-block transformer with a single attention head per block is as follows.\footnote{We describe a single attention head for simplicity of exposition. Our empirical results rely on multi-head attention.}
The tokens are first embedded into a $d$-dimensional space using the embedding/unembedding matrix $W_E$ as $Y_0 = X W_E$. The $L$ transformer blocks have the following form, defined recursively over the transformer blocks $i=0,1,2,\ldots,L-1$: 
\begin{align*}
    Y_0 & = X W_E , \;\;
    A_i = \attn_i(\lnorm(Y_i)) , \;\;
    Y_{i+1}  = Y_i + A_i + \mlp_i(\lnorm(Y_i + A_i))
\end{align*}
where\\[-1em]
\begin{itemize}[leftmargin=*]
    \item \lnorm \;\citep{ba2016layer} normalizes the per-example activations of a layer to have mean and standard deviation close to 0 and 1, respectively, and optionally rescales them by a learned parameter;
    \item $\mlp_i$ is a feed-forward network mapping its input as
    $\mlp_i(Y) = \max(0, Y W_{1}^{(i)} + b_{1}^{(i)})W_{2}^{(i)} + b_{2}^{(i)}$, 
where $W_{1}^{(i)}$, $b_{1}^{(i)}$, $W_{2}^{(i)}$, and $b_{2}^{(i)}$ are trainable;
\item $\attn_i$\; computes a linear projection of the inputs into query, key, and value matrices using trainable parameters $W_Q^{(i)}, W_K^{(i)}, W_V^{(i)}$, respectively. For each token, the query and key are used to compute a weighted combination of the values of the preceding tokens:
\begin{align*}
\attn_i(Y)
&= \sm(M \odot (YW_Q^{(i)} W_K^{(i)\top} Y^\top) / \sqrt{d}) {\color{blue}Y W_V^{(i)} W_O^{(i)} }
\end{align*}
where $M$ is a causal mask and $W_O^{(i)}$ is the output projection matrix.  We incorporate position information by applying RoPE  \citep{su2024roformer} to the query and key matrices.
\end{itemize}
\vspace{-0.5em}
The next-token probabilities are computed as
$p(\cdot | X) = \sm(\lnorm(Y_L) W_E^\top)$. 
This is an $n \times |\cV|$ matrix, whose last row gives the probability of the $(n+1)^{st}$ token given $x_1, \ldots, x_n$. Note that we assume that the embedding and unembedding matrices are tied ($W_E$ and $W_E^\top$, resp.). 

\paragraph{Copy attention heads}
As discussed earlier, copying symbols from the input to the output is often needed when reasoning about in-context information (such as in the logic tasks of Figure~\ref{fig:simple_logic}). Unfortunately, this can be challenging in the case of unseen tokens that models have been trained not to generate. 
Here we describe a minor architecture change that is helpful in copying both seen and unseen symbols. Let $A_j$ denote the $j^{th}$ row of a matrix $A$. Assuming that the rows of $W_E$ are almost orthogonal and almost normalized, a transformer will copy a token $x_{i}$ to position $j+1$ if $\lnorm (Y_{L})_j \approx W_{E, x_i}$. However, it may be difficult for the transformer to propagate an arbitrary input embedding $x_i$ through $L$ layers involving MLPs. We thus augment the architecture with a shortcut to the input embeddings, which we name \emph{copy attention}, implemented as follows:
\begin{align*}
    \cattn(Y) = \sm( M \odot (YW_Q W_K^\top Y^\top) / \sqrt{d})  {\color{red}X W_E }.
\end{align*}
Compared to standard attention, we replace the value and output projection $YW_V W_O$ with the input embeddings $X W_E$. Thus, for each $i$, the $i^{th}$ row of $\cattn(Y)$ will be a weighted combination of the input embeddings of tokens $x_1, ..., x_i$.
We now compute the next-token distribution as
\begin{align*}
    Z  &= \cattn(Y_{L-1}), \;\;
    p_{copy}(\cdot|X) = \sm(\lnorm(Y_L + Z) W_E^\top).
\end{align*}
Thus, ignoring layernorm, if the embeddings $W_E$ are approximately orthonormal and if the copy head sharply attends to a token, it will increase its output probability. Empirically we observe similar benefits of copy attention with and without output layernorm.
\todoa{\lnorm?}\todonev{repeated all experiments without lnorm}

The described copy attention is a modified version of the architecture change proposed by \citet{boixcan} for a single-layer transformer without MLPs, which outputs $\attn(X) = \sm(M \odot (XW_E W_Q W_K^\top W_E^\top X^\top) / \sqrt{d}) XW_E (W_V W_O + {\color{red}bI})$,  where $b$ is a trainable scalar and $I$ is the identity matrix. Effectively, this change adds the scaled input embeddings $bXW_E$ to the standard value-output parameters $W_V W_O$. The main difference in our implementation is that we use separate attention heads for copying (i.e. additional query and key parameters). This enables us to keep the MLPs in the standard attention blocks while not applying them to copy attention, and to trade-off between copying and standard generation more flexibly than via a single parameter. 

\section{Collapse of softmax weights for unseen labels}
\label{sec:dynamics}

As we will demonstrate, transformers struggle with symbolic reasoning that requires copying unseen tokens, and one of the reasons for this failure is the collapse of the unembeddings corresponding to those tokens. In this section we provide some insight into this phenomenon. Specifically, we show that once the model fits the data sufficiently well (to be specified next), each step of $\ell_2$-regularized gradient descent (GD) or stochastic gradient descent (SGD) will decrease the distance between the unembeddings of unseen tokens. 

In what follows, we consider the general problem of training a softmax classifier on a dataset $\cD = \{(x_n, y_n)\}_{n=1}^N$ of inputs $x \in \mathbb{R}^{d'}$ and labels $y \in [K]$. 
Let $\phi(\cdot|\theta): \mathbb{R}^{d'} \rightarrow B(r, d) $ be a neural network, parameterized by $\theta$, mapping the input features to a $d$-dimensional ball of radius $r$, $B(r, d) = \{\psi \in \mathbb{R}^d : \| \psi\| \leq r \}$.  We will use the shorthand $\phi_n := \phi(x_n | \theta)$ and $\phi_n^t := \phi(x_n | \theta^{(t)})$, where $\theta^{(t)}$ are the parameters at training step $t$. 
We consider a model in which the probability of a label is given by a softmax:
\begin{align}
\label{eq:smax}
 p(y|x, W, \theta) 
 = \frac{\exp(\phi(x|\theta)^\top W_{y})}{ \sum_{k=1}^K \exp(\phi(x|\theta)^\top W_{k})} 
\end{align}
where $W_k$ are the weights corresponding to class $k$. 
The parameters $(W, \theta)$ are optimized by minimizing the $\ell_2$-regularized negative log-likelihood with some regularization coefficient $\lambda>0$:
\begin{align*}
    L_{\lambda}(W, \theta | \cD) = -\frac{1}{N} \sum_{n=1}^N \ln p(y_n | x_n, W, \theta) + 0.5 \lambda \|W \|^2_F 
\end{align*}

The gradient of the  loss with respect to parameters $W_i$ is
  $  g_i := 
    \frac{1}{N} \sum_{n=1}^N (p(i|x_n, W, \theta) - I[i = y_n]) \phi_n + \lambda W_i$, 
where $I[i = y_n]$ is a label indicator function. 

\subsection{Contraction of unseen-label weights for GD/SGD}

Let $W^{(t)}$ denote the softmax weights at step $t$. Suppose that none of the examples in the dataset have the label $i$ or label $j$. The following lemma shows that in this case the parameter vectors $W_i$ and $W_j$ may contract:
\begin{lemma}\label{lemma:gd}
Suppose that we update the weights of the softmax classifier in Eq~\eqref{eq:smax} using gradient descent with learning rate schedule $\eta_t$. 
 If labels $i$ and $j$ are not present in the data, then the weights at time $t$ satisfy 
\begin{align*}
\left\|  W_j^{(t+1)} - W_i^{(t+1)} \right\| \leq \left(1 - \lambda \eta_t + \eta_t r^2 p_t \right) \left\|  W_j^{(t)} - W_i^{(t)} \right\|
\end{align*}
where $p_t := \max_{k \in \{i, j\}, n \in [N]}  p(k|x_n, W^{(t)}, \theta^{(t)})$.
\end{lemma}

The proof is given in Appendix~\ref{app:proof}. The same result holds for stochastic gradient descent by replacing the full batch with a minibatch, and having $p_t$ be the maximum probability of any unseen class for any minibatch example. We note that the result holds regardless of whether the parameters $\theta^{(t)}$ are fixed or changing over time, as long as the feature function $\phi(\cdot|\theta^{(t)})$ has bounded output. Thus, Lemma~\ref{lemma:gd} applies, for example, to transformers with layernorm in the last layer. 

Lemma~\ref{lemma:gd} implies that the weights for a pair of unseen labels will contract whenever $r^2 p_t < \lambda$. Furthermore, if our learning rate $\eta_t$ is bounded away from zero, and assuming that the classifier has sufficient capacity to fit the data so that the probability of unseen labels $i, j$ approaches zero for training examples, we see the corresponding weight entries collapse.

Thus we have shown that a step of $\ell_2$-regularized GD/SGD will contract the last-layer weights of unseen labels if the input features have bounded norm and the model fits the data well, even if the features are changing over time. This setting corresponds to modern deep neural networks which overfit training data and use layer normalization to ensure boundedness of the learned features.

\subsection{Empirical collapse of unseen-token (un)embeddings under AdamW}

While our analysis considers SGD, in practice transformers are often trained using the AdamW optimizer \citep{loshchilov2017fixing}, for which obtaining similar results is more involved. In this section we empirically demonstrate that similar collapse also happens when training transformers using AdamW. In particular, we train transformers on logic problems using a vocabulary with 100 deliberately heldout tokens (see Section~\ref{sec:experiments} for full experimental details),  using AdamW with different regularization strengths $\lambda$, with and without layernorm in the last layer. Figure~\ref{fig:cossim_hparam} shows the mean cosine similarity between the (un)embeddings of seen tokens, and between the (un)embeddings of unseen tokens.
We observe that unseen tokens are much more correlated than seen tokens, especially in the presence of layer normalization, even without regularization ($\lambda=0$).

\begin{figure}[t]
    \centering    \includegraphics[width=0.6\textwidth,trim={0 0 20.6cm 0},clip]{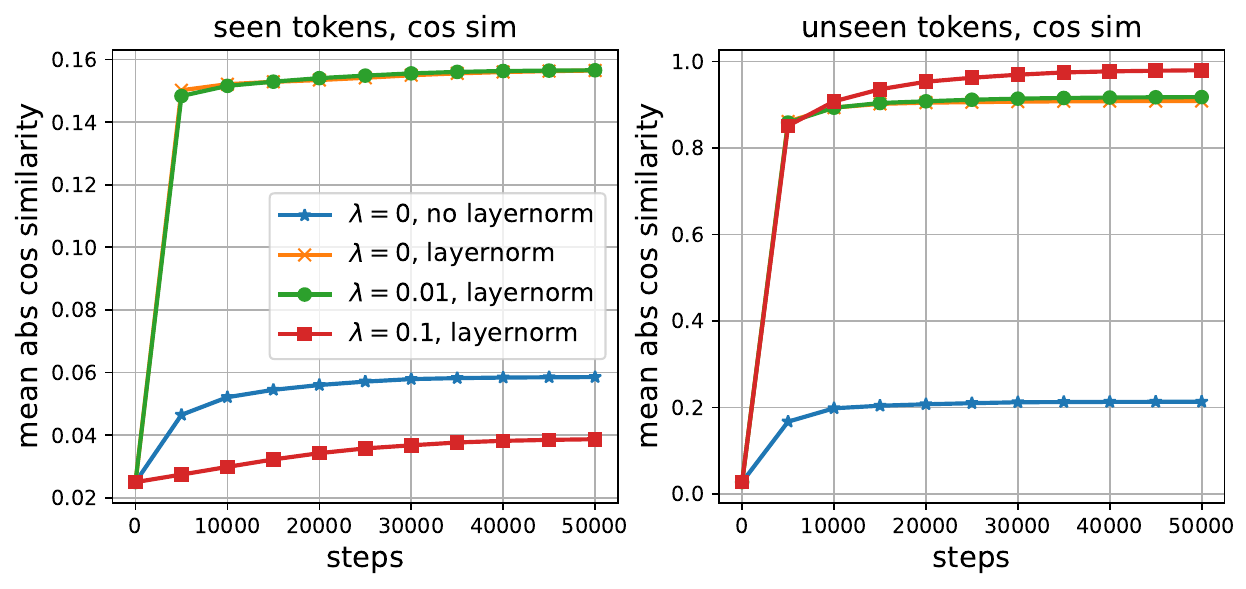}
    \vspace{-.1em}
    \caption{Mean pairwise cosine similarities of the unembeddings for different hyperparameter choices (note the different scale of the vertical axis in the two graphs). 
    In the presence of layernorm, the unembeddings of unseen tokens collapse to nearly the same-direction vector (also note that the actual magnitudes are also very close). This is consistent with Lemma~\ref{lemma:gd} which requires the inputs to the softmax function to be bounded (i.e., normalized). Collapse is more severe for higher regularization $\lambda$.
    }
    \label{fig:cossim_hparam}
    \vspace{-.2em}
\end{figure}


\section{Experiments on symbolic propositional logic}
\label{sec:experiments}

In this section we study how transformers can learn to reason in the context of propositional logic problems, and demonstrate the issues and solutions highlighted earlier about generalization to unseen symbols.

\paragraph{Propositional logic problem}  
We create synthetic propositional logic problems involving only definite clauses (also known as Horn clauses). 
The goal is to determine the truth value of a query predicate given a set of rules (definite clauses, e.g., 'If A then B') and facts (true zero-order predicates, e.g., 'A is true').  We generate random synthetic problems similarly to the rule-priority recipe of \citet{zhang2023paradox}, but additionally include chain-of-thought (CoT) reasoning that can be used to derive the answer. 
The CoT is generated by running the \emph{forward-chaining} algorithm, which iteratively applies the \emph{modus ponens} rule to derive new facts until either the query predicate is proven or no further predicates can be proven true. If a predicate cannot be proven true, it is false. See Figure~\ref{fig:simple_logic} for an example and data format. For logic problems involving only Horn clauses, modus ponens is known to be sound and complete: any predicate entailed by the rules and facts can be derived by iteratively applying modus ponens.    
The generated examples contain up to 20 predicates per question and up to 40 rules per question.  We ensure balanced labels by selecting a true predicate as the query with probability 0.5.
\vspace{-.5em}

\paragraph{Model} We train a small 10-layer decoder-only model (60M-75M parameters, depending on the number of predicates) from scratch on this data using the NanoDO library \citep{nanodo} and the AdamW optimizer \citep{loshchilov2017fixing} with weight decay $\lambda=0.001$.   Rather than using a fixed dataset, we continually generate fresh examples on-the-fly during training. We only train the model to predict the tokens corresponding to the reasoning and the answer, that is, we mask out the question tokens.\footnote{In our initial experiments, we found that training the model to predict the randomly-generated questions in addition to answers led to worse performance and increased hallucinations in the CoT.} We experiment both with standard architectures and copy attention, using a custom vocabulary to control the tokenization of the problems. 
See Appendix~\ref{sec:appendix_experiments} for further details on the data, architecture, and training. 

\begin{figure*}[t!]
    \centering
    \includegraphics[width=0.95\textwidth]{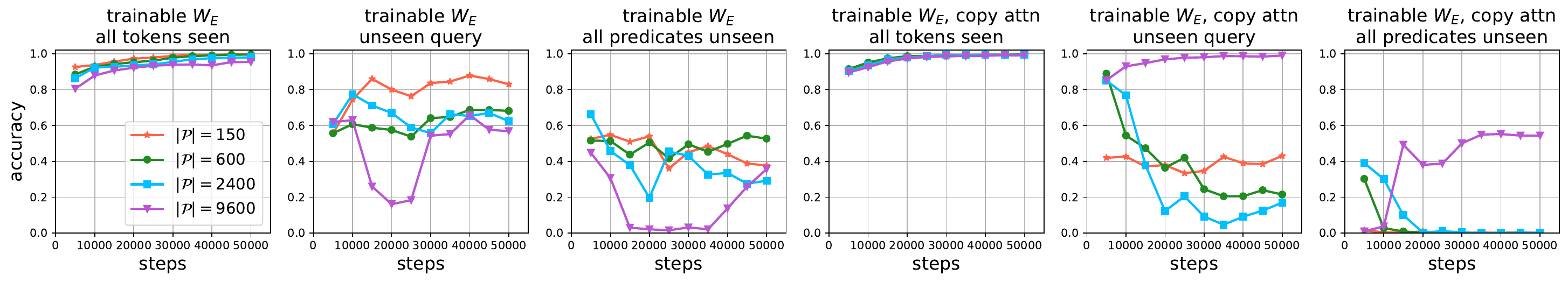}
     \includegraphics[width=0.95\textwidth]{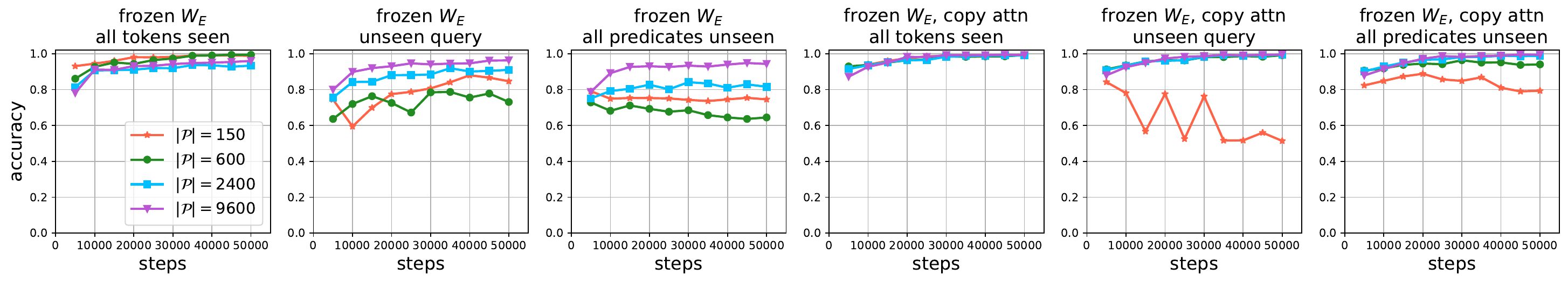}
    \includegraphics[width=0.95\textwidth]{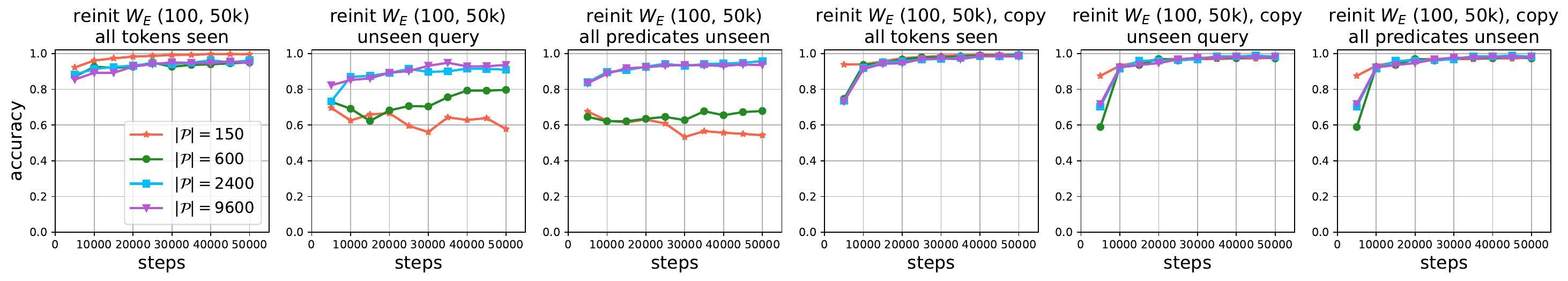}
         \includegraphics[width=0.95\textwidth]{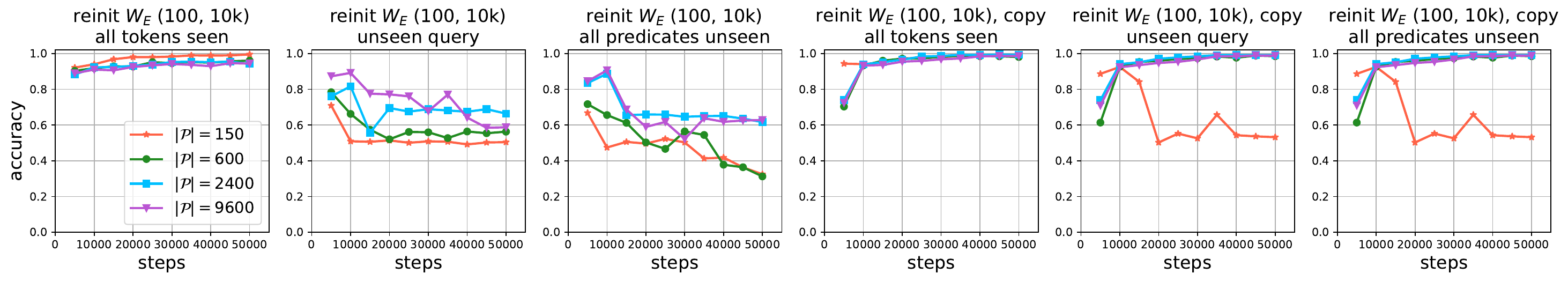}
    \caption{Propositional logic evaluation accuracy for models trained on single-token symbols. 
    We evaluate accuracy with (1) all tokens seen, (2) query unseen, (3) all predicates unseen. The rows correspond to the unembeddings being (1) trainable (2) frozen (3) periodically reinitialized  every 100 steps throughout training, and (4) periodically reinitialized every 100 steps for the first 10k steps. We train each model with and without copy attention (right and left side of the figure).}
    \label{fig:single_token_nodelim}
\end{figure*}

\paragraph{Experiment settings} In this setting, each logic predicate corresponds to a single token. The predicates for each problem are selected uniformly at random from a set $\cP$ (and thus the size of $\cP$ serves as a proxy for symbolic diversity). The vocabulary for each training task includes 100 unused tokens which are subsequently used as predicates during evaluation. 
We train models \emph{with and without copy attention}, and with \emph{trainable, frozen, and periodically reinitialized (un)embeddings}. We vary the size of the training predicate vocabulary $\|P\| \in \{150, 600, 2400, 9600\}$ as a proxy for diversity. We evaluate the accuracy of the trained models on a randomly-generated set of 2048 examples with (1) all tokens seen, (2) query unseen, and (3) all predicates unseen.

\paragraph{Results and discussion} The results are shown in Figure~\ref{fig:single_token_nodelim}. 
All model versions perform well on in-distribution data (seen tokens). The standard transformer fails when one or more symbols are replaced by unseen tokens. Examining outputs reveals that this is because the model learns not to generate the unseen token (needed for reasoning). The model instead either ignores the token or replaces it with a seen predicate (see Figure~\ref{fig:cot_single_token_vanilla} for some examples of generated CoT in this case).
Adding copy attention and increasing diversity results in good performance on the evaluation with a single unseen token. In fact, early in the training, models with vocabulary size 600, 2400, and 9600 all generalize well and copy the unseen token in their CoT, but only the model with $|\cP|=9600$ keeps this inductive bias. This result is consistent with existing works on in-context learning (ICL), which show that a certain diversity threshold is necessary for generalization \citep{raventos2023pretraining, he2024learning}, and that ICL is transient \citep{chan2022data,panwarcontext,chantoward2025}.
Copy attention fails in the case where all symbolic variables are replaced by unseen tokens, regardless of diversity. This is explainable at least in part by the collapse of the unseen (un)embeddings, which makes these variables difficult to distinguish.

Freezing the embeddings to their initialization during training (which are approximately orthogonal and hence well-distinguishable for different tokens) allows all models to generalize to unseen tokens, with the copy-attention models achieving higher accuracy (or having a lower diversity threshold).  Similarly, \emph{active forgetting} \citep{chen2023improving} (reinitializing embeddings every 100 steps) generalizes to unseen tokens given sufficient diversity and/or an architecture with copy heads.\footnote{We tried reinitializing every 100 and every 1k steps, for 10k steps or throughout training; see Appendix~\ref{sec:active_forgetting} for full experiments.} 
These results further support our hypothesis that symbolic reasoning is impeded by learning not to emit unseen tokens and the unembedding collapse.\footnote{We have also tried freezing / reinitializing only the embeddings of the predicate tokens. Contrary to our intuition, this yielded worse generalization compared to freezing / reinitializing all embeddings, and requires further investigation.} 
Unfortunately, freezing and reinitializing are somewhat unsatisfactory solutions for general-purpose models, which typically benefit from trainable embeddings.\footnote{See also Figure~\ref{fig:c4} in the Appendix~\ref{app:c4} for some empirical evidence in this regard, where we compare models with trainable and frozen embeddings on the C4 corpus \citep{raffel2020exploring}.}  
Finally, we evaluated \emph{temporary active forgetting} \citep{ananddual} (reinitializing for the first 10k steps of the training). We observed that it does not generalize well for the standard architecture (no copy attention) -- performance collapses once reinitialization stops. However, it does generalize well with the copy attention architecture for all but the smallest vocabulary $\cP$. While the unseen unembeddings again collapse in this setting (having a cosine similarity of about 0.9), the model with copy attention appears to attend to tokens sharply enough to distinguish different symbols.

\section{Experiments on Gemma3}

In this section, we investigate whether our findings related to the collapse of embeddings hold for larger (open-weights) models, namely the Gemma~3 series \citep{team2025gemma}. These models use tied embedding and unembedding parameters, as assumed throughout. 
\vspace{-0.5em}

\paragraph{Collapse of unused tokens} The Gemma~3 tokenizer vocabulary is of size 262,144 and includes 99 deliberately unused tokens. For each instruction-tuned model size (1B, 4B, 12B, 27B), with embedding dimensions (1152, 2560, 3840, 5376) respectively, we compute the cosine similarities of the embeddings of the unused tokens and a randomly chosen subset of 100 used tokens. The results are shown in Figure~\ref{fig:unused}: 
in all models, the unused tokens have a higher cosine similarity than the used ones. The mean cosine similarities are (0.78, 0.33, 0.23, 0.24) for the unused tokens and (0.09, 0.03, 0.02, 0.01) for the used tokens.%
\footnote{For the pretrained models, the mean cosine similarities are (0.79, 0.43, 0.33, 0.29) for the unused tokens and (0.06, 0.03, 0.02, 0.01) for the used tokens}
Thus the effect of the collapse is smaller for larger models.  
\vspace{-0.5em}

\paragraph{Finetuning with unused tokens} To investigate whether the collapse of unused tokens is an issue for downstream applications relying on such tokens, we finetune Gemma~3 1B IT on propositional logic data constructed similarly as before, with 80 symbols (predicates) corresponding to (1) unused tokens, (2) English adjectives, and (3) ASCII lowercase and uppercase letters and repeated letters (e.g., 'aa'). We verify that all symbols correspond to single tokens. We finetune all parameters using AdamW with a constant learning rate $5 \cdot 10^{-5}$ and parameters $\beta_1=0.9$, $\beta_2=0.999$, $\epsilon=10^{-8}$, $\lambda=10^{-4}$, using batch size of 128. We show evaluation accuracy on a holdout dataset of 512 examples (with the same predicates as the trained model) in Figure~\ref{fig:gemma_sft_const} for 4 random training runs. We observe that training using unused tokens can be slow: the runs using adjectives and letters as symbols reach $\sim$90\% accuracy in under 500 steps, while the runs involving unseen tokens take about 5000 steps to reach similar performance. The runs involving letters are the most stable, perhaps because these are more commonly used as variable names in the training data.
\vspace{-0.5em}

\begin{figure}[t!]
    \centering
    \includegraphics[width=0.8\textwidth]{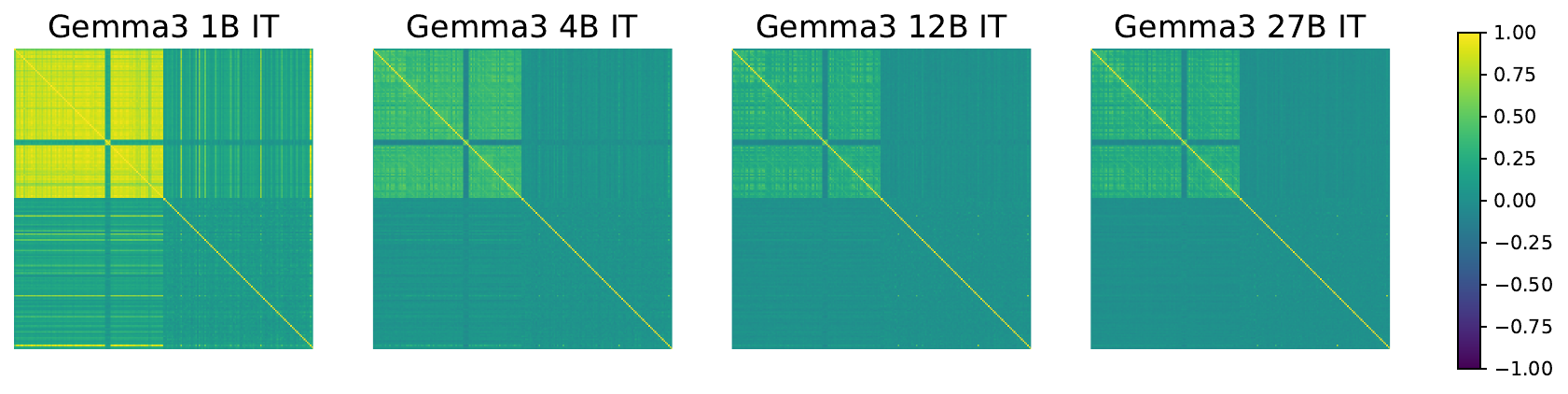}
    \vspace{-0.7em}
    \caption{Cosine similarities between the embeddings of the 99 unused tokens and 100 randomly chosen tokens. In all models, the unused tokens are correlated more than the used ones, with the effect more pronounced the smaller the embedding dimension is. }
    \label{fig:unused}

\vspace{0.7em}

    \centering
    \includegraphics[width=0.6\textwidth]{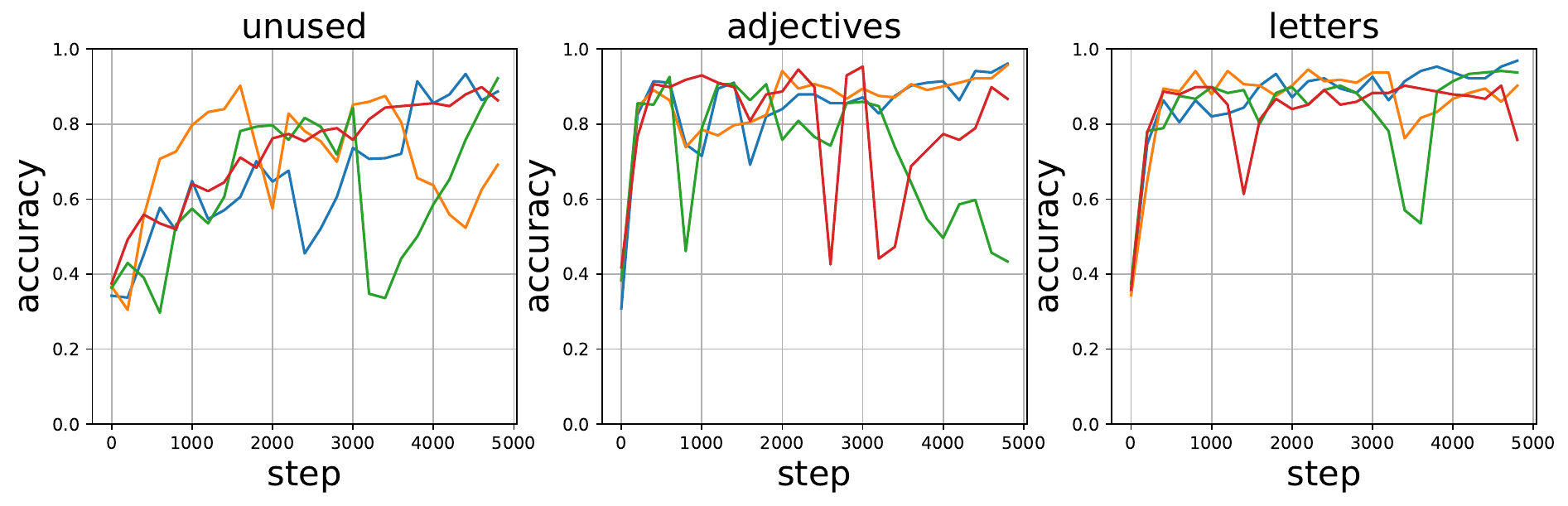}
    \vspace{-0.5em}
    \caption{Accuracy when finetuning Gemma~3 1B IT on logic problems with predicates corresponding to unused tokens, English adjectives, and letters. Curves correspond to random seeds. With adjectives and letters, accuracy reaches ~90\% in about 500 finetuning steps, whereas with unused tokens reaching 90\% accuracy requires about 5000 steps, which may be an artifact of embedding collapse.}
    \label{fig:gemma_sft_const}
    \vspace{-0.3em}
\end{figure}

\paragraph{Collapse of rare tokens during instruction tuning}
Next we investigate whether collapse of embeddings occurs during conventional instruction tuning of models. We compute the singular value decomposition (SVD) of the (un)embedding matrices of Gemma~3 models, normalized by their Frobenius norm. In Figure~\ref{fig:spectra}, we plot the difference between the sorted singular values of IT (instruction-tuned) and PT (pre-trained) models. In all models, the IT version places more mass on the top two singular values than the corresponding PT version, with the effect most pronounced in the 1B version.
We observe that tokens whose cosine similarity increases between PT and IT versions of Gemma~3 1B often correspond to html tags and non-Latin script symbols; for example the cosine similarity between {\bf </h6>} and {\bf );">} increases from -0.044 to 0.273. \todonev{there are examples with bigger changes, but don't know how to render them...} Thus it is plausible that these tokens are not included or rarely seen in the instruction tuning data, though we cannot definitively verify this claim. We speculate that this increase in correlation does not meaningfully degrade capabilities. 




\vspace{-0.3em}

\section{Discussion}

\vspace{-0.3em}

We have investigated the ability of transformers to reason with abstract symbols represented using previously-unseen tokens. We found that generalization to unseen tokens is hindered by an inductive bias that causes their (un)embeddings to collapse, and proposed methods for addressing this collapse. Our work resolves the issues with architectural interventions proposed by \citet{boixcan} in the presence of multiple unseen variables, and also helps explain the benefits of \emph{temporary active forgetting} \citep{ananddual}. 
We have also observed some evidence of collapse in the open-weight Gemma~3 models, and found that finetuning these models is considerably slower with unseen tokens.  
Some limitations of our work are that the theoretical analysis of the collapse includes additional assumptions and that most of our experiments are on small models and only on propositional logic problems. Interesting directions for future work include strategies for improving finetuning with collapsed unused tokens, as well as investigating reasoning with multi-token symbols (a preliminary study on the latter is given in Appendix~\ref{app:multi}).

\bibliography{ref}
\bibliographystyle{colm2026_conference}

\appendix
\newpage

\section{Proof of Lemma~\ref{lemma:gd}}
\label{app:proof}

\begin{proof}
Let $p_{i,n}^t = p(i |x_n, W^{(t)}, \theta^{(t)})$. 
The weights satisfy

\begin{align}
 W_j^{(t+1)} -  W_i^{(t+1)} =  (1 - \lambda \eta_t)\left(W_j^{(t)} - W_i^{(t)}\right) - \eta_t \frac{1}{N} \sum_{n=1}^N \left(p_{j,n}^t - p_{i,n}^t\right) \phi_n^t  \label{eq:wj}
\end{align}

By the Lagrange mean-value theorem, $\exp(b) - \exp(a) = \exp(c)(b - a)$ for some $c$ such that $a < c < b$. Assume without loss of generality that $\phi_n^{t\top} W_j^{(t)} \geq \phi_n^{t\top} W_i^{(t)}$. Setting $a = \phi_n^{t\top} W_i^{(t)} - \ln Z_t(\phi_n^t)$ and $b = x_n^{t\top} W_j^{(t)} - \ln Z_t(\phi_n^t)$, where $Z_t(\phi_n) = \sum_y \exp(\phi_n^\top W_y^{(t)})$ is the normalizer, we have that for some $p_{n}^t$ such that $\min(p_{i,n}^t, p_{j,n}^t) \leq p_n^t \leq \max(p_{i,n}^t, p_{j,n}^t)$
\begin{align}
   p_{j,n}^t - p_{i,n}^t &= \exp(\phi_n^{t\top} W_j^{(t)} - \ln Z_t(\phi_n^t)) - \exp(\phi_n^{t\top} W_i^{(t)} - \ln Z_t(\phi_n^t)) \nonumber \\
  &= p_n^t  \phi_n^{t\top} \left(W_j^{(t)} - W_i^{(t)}\right)
  \label{eq:delta}
\end{align}

Let $p_t = \max_{n} p_n^t$.  Plugging \eqref{eq:delta} into \eqref{eq:wj} and using the fact that inputs lie in a ball of radius $r$ gives the result.  
\end{proof}


\section{Implementation details}
\label{sec:appendix_experiments}

\paragraph{Architecture and training.} Our implementation of transformer training and evaluation builds upon the NanoDO framework \citep{nanodo}.  For the simple logic experiments, we train a decoder-only transformer with 10 layers and embedding size 1024 split across 8 heads. With copy attention, we add a copying block with 8 heads. The architecture is as described in Section~\ref{sec:architecture}. 
We use MLPs with hidden size 1024 and GELU activations \citep{hendrycks2016gaussian}. We use RoPE positional embeddings \citep{su2024roformer}.  We train using the AdamW optimizer \citep{loshchilov2017fixing} with decay 0.001, batch size 256, peak learning rate 0.0001, warmup and final rate of 0.00001, and 2000 warmup steps. We clip gradients by global norm 1.

We use a custom tokenizer designed for logic data. For the single-token symbol experiments, the vocabulary includes $|\cP| \in \{150, 600, 2400, 9600\}$ zero-order predicates, the special tokens shown in Figure~\ref{fig:special_tokens}, pad token, and 100 additional unused tokens. The vocabulary size is further padded by unused tokens to a multiple of 4.  
For the multi-token symbol experiments, the vocabulary includes 26 baseline tokens and sequences of these tokens form predicates. For each predicate in each example, we sample its length $i \in [n]$ and tokens uniformly at random, while ensuring there are no duplicates.

\begin{figure}[h!]
\begin{framed} 
{ Facts EndFacts Rules EndRules Query EndQuery Answer Reasoning EndReasoning Newfact EndNewfact If then EndIf yes no and ? <BOS> <EOS> <PAD> <UNK>}
 \end{framed}
    \caption{Special tokens in the vocabulary in addition to zeroth order predicates.}
    \label{fig:special_tokens}
\end{figure}

\begin{figure}[t!]
\begin{framed} 
{\color{blue} <BOS> {\bf Rules:}
 If A and B then C EndIf 
If B then A EndIf \newline
 {\bf Facts:} B EndFacts \newline 
 {\bf Query:} C ? } \newline
{\color{magenta} {\bf Reasoning:} \newline
 Facts: B EndFacts 
 If B then A EndIf 
 Newfact A EndNewfact \newline
 Facts: B and A EndFacts
  If A and B then C EndIf 
  Newfact C EndNewfact \newline
 Facts: B and A and C EndFacts} \newline
 {\color{teal} {\bf Answer:} yes EndAnswer <EOS>}
 \end{framed}
    \caption{Formatting example for a propositional logic problem with reasoning used for training. The separator 'and' is only used in the multi-token symbol experiments.}
    \label{fig:simple_logic_format}
\end{figure}

\paragraph{Data.} We generate propositional logic problems similarly to the rule-priority (RP) recipe described in \citet{zhang2023paradox}. RP problems are generated by randomly sampling a subset of predicates, and randomly sampling facts and rules using those predicates. The facts and rules are then used to compute the truth values. To balance label probabilities, we select the query to be a true predicate with probability 0.5. 
The problems we generated had between 5 and 20 predicates and between 0 and 40 rules. The examples presented to the model are formatted as in Figure~\ref{fig:simple_logic_format}. We generated reasoning traces of the form in Figure~\ref{fig:simple_logic_format} by running forward chaining until either the query predicate was proven true or no further predicates could be proven true, and generating text corresponding to the algorithm trace.

\section{Additional experiments with temporary active forgetting}
\label{sec:active_forgetting}

We evaluate the \emph{temporary active forgetting} approach of \citet{ananddual} on logic problems with single-token symbols. Here the (un)embeddings are periodically reinitialized every $k_1$ steps for the first $k_2$ steps of training. We consider the following values of $(k_1, k_2)$: $(100, 10k)$,  $(1k, 10k)$, $(100, 50k)$, $(1k, 50k)$. Note that we train for 50k steps in total, so some of the settings correspond to the (non-temporary) active forgetting approach of \citet{chen2023improving}. The results are shown in Figure~\ref{fig:active_forgetting}. We observe that temporary active forgetting does not load to the desired inductive bias (symbolic reasoning), unless combined with copy attention. Active forgetting throughout the training can generalize with the vanilla architecture, provided sufficiently high symbolic diversity. 

\begin{figure}[t!]
    \centering
    \includegraphics[width=\textwidth]{fig_single_reinit_100_10k.pdf}
        \includegraphics[width=\textwidth]{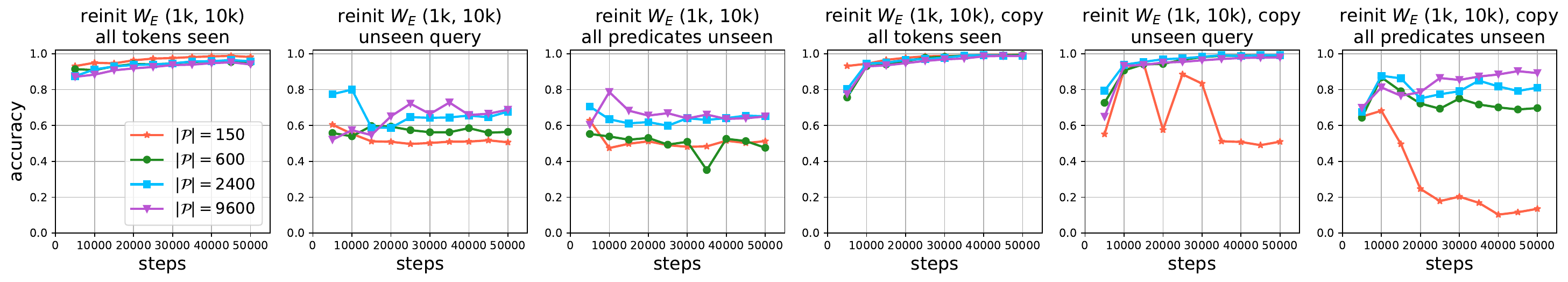}
     \includegraphics[width=\textwidth]{fig_single_reinit_100_50k.pdf}
    \includegraphics[width=\textwidth]{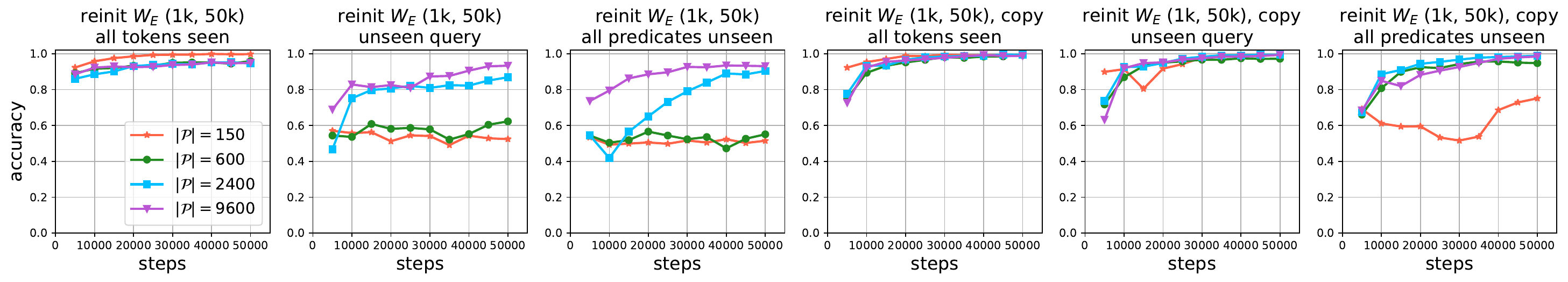}
    \caption{Propositional logic evaluation accuracy for models trained on single-token symbols with (temporary) active forgetting.}
    \label{fig:active_forgetting}
\end{figure}

\section{Examples of model samples}

Figure~\ref{fig:cot_single_token_vanilla} shows samples from the vanilla transformer trained on logic when the query (single-token) predicate is replaced by an unseen token. If the unseen token appears in rules, the model often replaces it by a seen token. If the unseen token appears in facts, the model often ignores it. Most model mistakes are false negatives. 

\begin{figure}[t!]
\small
\begin{subfigure}{.45\textwidth}
\begin{framed} 
{\bf Rules:} \newline
If enchanting
then UNK
EndIf \newline 
If perfect
then silly 
EndIf \newline 
[...] \newline
EndRules \newline
{\bf Facts:} enchanting
EndFacts \newline
{\bf Query:} UNK
EndQuery
? \newline
{\bf Reasoning:} \newline
Facts: enchanting
EndFacts \newline
If enchanting
then perfect
EndIf \newline
Newfact: perfect
EndNewfact \newline
Facts: enchanting perfect
EndFacts \newline
If enchanting
then ugliest
EndIf \newline
Newfact: ugliest
EndNewfact \newline
Facts: enchanting perfect ugliest
EndFacts \newline
If perfect
then silly
EndIf \newline
Newfact: silly 
EndNewfact \newline
Facts: enchanting perfect ugliest silly
EndFacts \newline
EndReasoning \newline
{\bf Answer:} no
EndAnswer
 \end{framed}
 \end{subfigure}
 \begin{subfigure}{.45\textwidth}
\begin{framed} 
{\bf Rules:} \newline [...] \newline
EndRules \newline 
{\bf Facts:} hilarious perfect bored proud UNK
EndFacts \newline 
{\bf Query:} UNK
EndQuery
? \newline
{\bf Reasoning:} \newline
Facts: hilarious perfect bored proud \newline
EndFacts
EndReasoning \newline
{\bf Answer:} no
EndAnswer
 \end{framed}
 \end{subfigure}
    \caption{Examples of reasoning traces in the case of a single single-token symbolic variable. The model refuses to generate the unseen token in its reasoning trace, leading to CoT errors and eventually the wrong answer. Most mistakes occur when the true answer is 'yes' and the model outputs 'no'. }
    \label{fig:cot_single_token_vanilla}
\end{figure}

\section{Effect of freezing embeddings and copy attention on language modeling}
\label{app:c4}
\begin{figure}[ht!]
    \centering
    \includegraphics[width=0.95\textwidth]{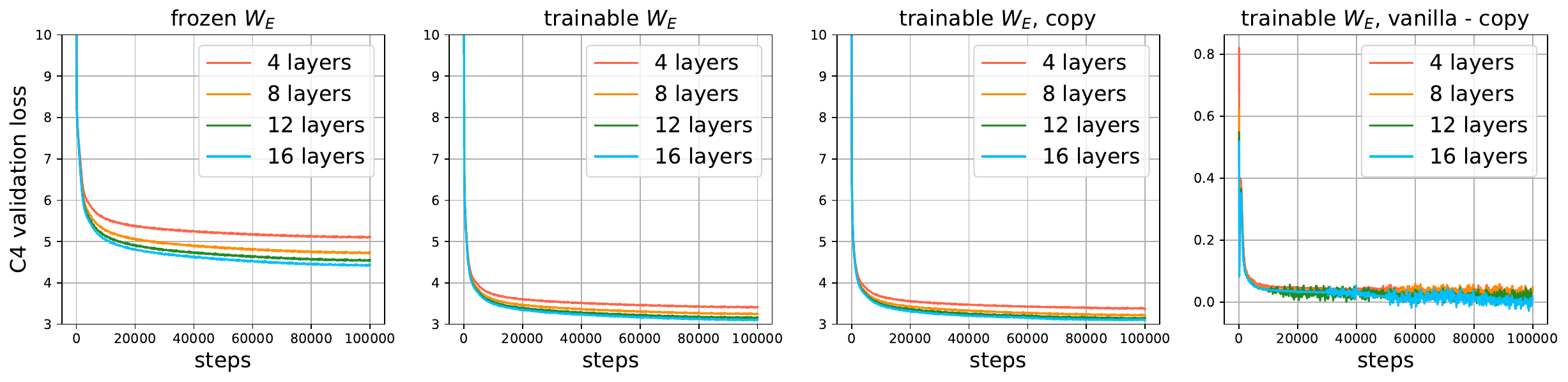}
    \caption{Evaluation loss on the C4 dataset with trainable and frozen embeddings and with copy attention. Freezing embeddings results in significantly higher validation loss compared to using trainable embeddings. Adding copy attention results in slightly smaller loss compared to the default architecture. To make the difference between the second and third figure visible, the last figure shows the difference in the validation loss of the second standard (vanilla) architecture and the one with the copy attention.}
    \label{fig:c4}
\end{figure}

In this section we train small transformers with 4, 8, 12, and 16 layers on the C4 dataset \citep{raffel2020exploring}, with trainable and frozen embeddings and with copy attention. We show the validation loss for all models in Figure~\ref{fig:c4}. Freezing embeddings leads to considerably worse evaluation loss and thus may not be a practical solution. Interestingly, copy attention has slightly better validation loss, especially early in the training and for small models.

\iftrue
\section{Multi-token symbols}
\label{app:multi}

\begin{figure}[h!]
\begin{framed} 
{\bf Rules:} If qwert, asdf then {\color{blue} abcd}e. If zxcv then {\color{blue} abcd}. 
{\bf Facts:} qwert, asdf. 
{\bf Query:} {\color{blue} abcd} ?
 \end{framed}
    \caption{Stylized example where the multi-token query shares a prefix with another predicate.}
    \label{fig:multitoken_prefix}
\end{figure}

\begin{figure}[ht!]
    \centering
   \includegraphics[width=0.9\textwidth]{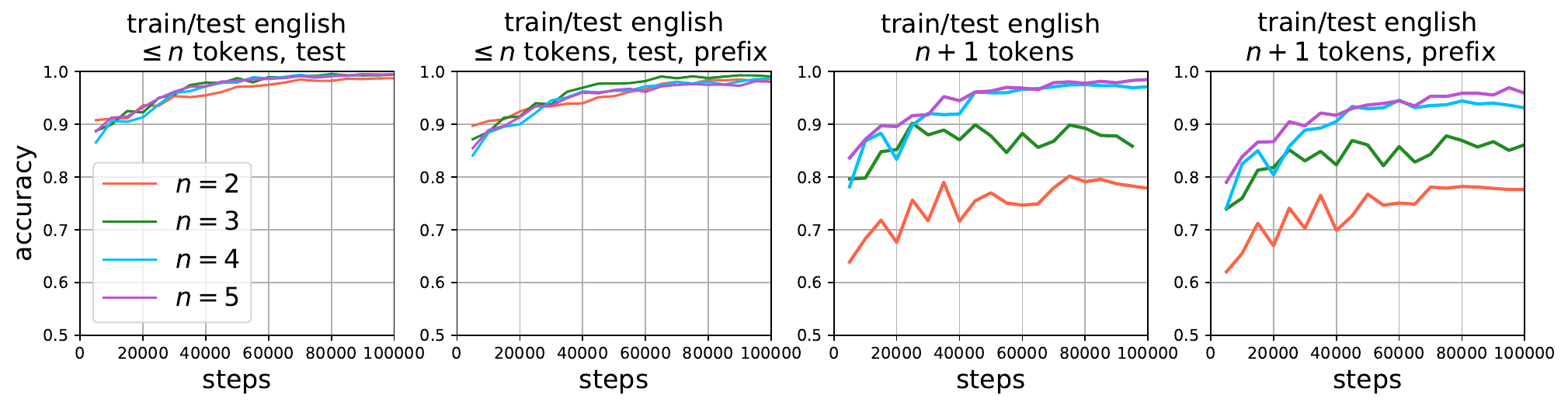}
    \caption{Evaluation of models trained on symbols of up to $n$ tokens on problems with symbols of length up to $n$, and length $n+1$. 'prefix' in figure title means that the query shares a prefix with another predicate, making them more easily confusable.
    }
    \label{fig:multi_token_unif}
\end{figure}

While our work focuses primarily on unseen tokens, here we also conduct a small preliminary study on unseen \emph{multi-token variables} (where all tokens are seen, but not all combinations are seen). This case has not been studied systematically in literature to the best of our knowledge, but it is related to much existing works on perturbing reasoning benchmarks by changing names and numerical values and using synonyms \citep{jiang2024peek,mirzadehgsm}. While embedding collapse is no longer an issue here, we observe other problems, including length generalization and difficulty in distinguishing similar symbols (such as those sharing a prefix), which we describe next.

To evaluate symbolic reasoning with multi-token symbols, we train models with symbols of length up to $n$ tokens for $n \in \{2, 3, 4, 5\}$. We select the number of tokens for each symbol in a problem uniformly at random. We use $b=26$ base tokens (corresponding to lowercase English alphabet letters). For each length 1, ..., $n$, we generate up to $m=1000$ symbols as follows.
Using a corpus of 10K most frequent English words \citep{10kwords}, we estimate the probability of the first letter and the next letter given the current letter, using counts regularized with a pseudo-count of 1 (i.e., the Laplace estimator). We use these to generate new random words (symbols). 

We split the symbols of each length evenly into train and test sets. We ensure that the train two-token symbols contain all tokens; thus, \emph{there are no unseen tokens or collapsed embeddings in this experiment}, eliminating the problems stemming from the undesired behavior of the embeddings. 
We evaluate models on the test symbols of length up to $n$ and on new symbols of length $n+1$. Additionally, to make evaluation more adversarial, we create examples where the query shares a prefix with the longest non-query predicate (e.g., 'abcde' and 'abcdf'), making the two more easily confusable. 

The results are shown in Figure~\ref{fig:multi_token_unif}. While models generalize well to problems involving test symbols of lengths seen during training, performance gets worse for symbols that are longer by even one token. By inspecting the samples containing errors, we observe the following qualitative behaviors. 
Models trained with $n=2$ and $n=3$ tend to truncate symbols to $n$ tokens; these errors also occur for higher $n$, but less frequently. This finding may be of practical interest, as typical math datasets involve short variable names that are unlikely to be tokenized to more than 3 tokens, and similar truncations have also been observed in large state-of-the-art models \citep{malek2025frontier}. 
For $n \geq 4$, accuracy is lower on $n+1$-token symbols when the query shares a prefix with another variable, suggesting that their representations start to collapse (this phenomenon was studied in more detail by \citet{barbero2024transformers}). 
Models evaluated on 5-token or 6-token symbols also occasionally \emph{hallucinate} -- go into a loop repeating a set of symbols. We speculate that this behavior is induced by long unlikely sequences of tokens. 

 Thus, while using multi-token variables and using the same symbol tokens at train and test time resolves some of the issues studied in our work (specifically, embedding collapse and generating unseen tokens), it also comes with its own set of issues requiring different interventions, and is an interesting direction for future work.


\fi

\section{Additional Gemma experiments}

\begin{figure}[ht!]
    \centering
    \includegraphics[width=0.95\textwidth]{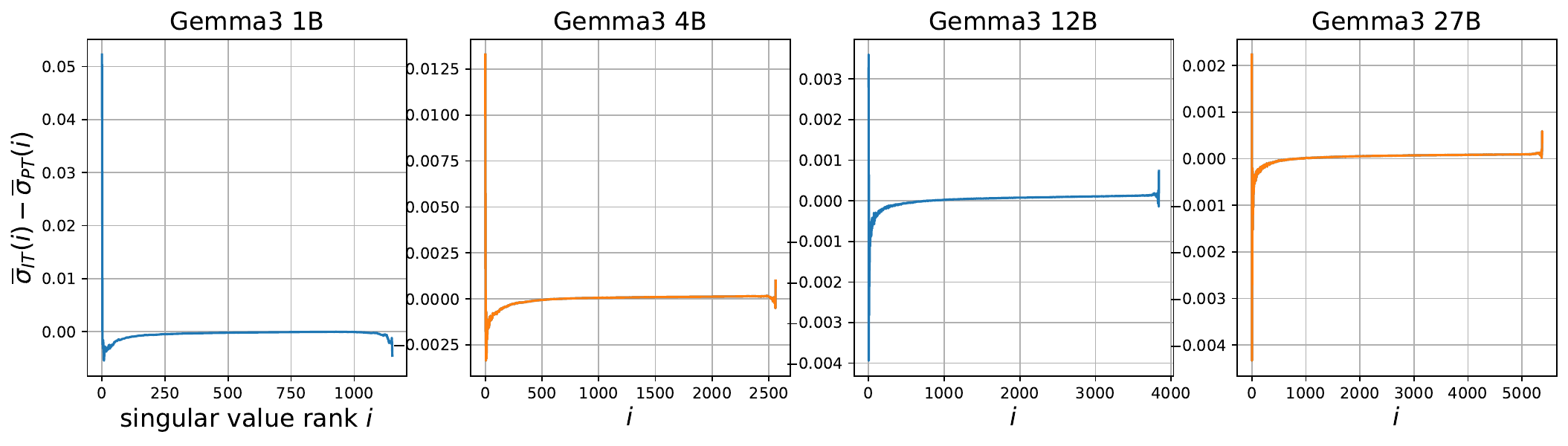}
    \caption{
    Difference between the sorted singular values of the Frobenius-normalized (un)embeddings of instruction-tuned (IT) and pre-trained (PT) Gemma~3 models. We observe that the 1B, 4B, and 12B IT models have larger top two singular values than the PT versions, suggesting some collapse during the IT process. The 27B model has a larger top singular value. The effect is generally weaker in larger models which have a larger embedding dimension.}
    \label{fig:spectra}
\end{figure}

\end{document}